\documentclass[10pt, a4paper, onecolumn]{article}
\usepackage{framed,multirow}
\usepackage[utf8]{inputenc}
\usepackage{url}
\usepackage{xcolor}
\definecolor{newcolor}{rgb}{.8,.349,.1}

\usepackage{amssymb}
\usepackage{latexsym}
\usepackage[cmex10,sumlimits]{amsmath}
\usepackage{amsthm}
\usepackage[noend]{algorithmic}
\usepackage{algorithm}
\usepackage{caption}
\usepackage{subcaption}
\usepackage{graphicx}

\newtheorem{definition}{Definition}

\newtheorem{proposition}{Proposition}

\DeclareMathOperator*{\maximize}{maximize}
\DeclareMathOperator*{\minimize}{minimize}

\newcommand{\fun}[1]{{\mathrm{#1}}}
\newcommand{\pp}[1]{{\left( #1 \right)}}

\newcommand{\ebr}[1]{{\left[ #1 \right]}}
\newcommand{\norm}[1]{{ \Vert #1 \Vert }}
\newcommand{\snorm}[1]{{ \Vert #1 \Vert^2 }}
\newcommand{\abs}[1]{{ \left| #1 \right| }}

\newcommand{\T}{\top}
\newcommand{\mat}[1]{\boldsymbol{#1}}
\newcommand{\prob}[1]{\mathrm{Pr}\left[#1\right]}

\newcommand{\f}{f_\phi}
\newcommand{\h}{h_\phi}

\title{Robust Explanations for Private Support Vector Machines}

\author{
    Rami Mochaourab\footnote{Digital Systems Division, RISE Research Institutes of Sweden, \texttt{rami.mochaourab@ri.se}} \and Sugandh Sinha\footnote{Digital Systems Division, RISE Research Institutes of Sweden, \texttt{sugandh.sinha@ri.se}} \and Stanley Greenstein\footnote{Dept. of Law, Stockholm University, \texttt{stanley.greenstein@juridicum.su.se}} \and Panagiotis Papapetrou\footnote{Dept. of Computer and Systems Sciences, Stockholm University, \texttt{panagiotis@dsv.su.se}}
    }

\date{}

\begin{document}
\maketitle

\begin{abstract}
We consider counterfactual explanations for private support vector machines (SVM), where the privacy mechanism that publicly releases the classifier guarantees differential privacy. While privacy preservation is essential when dealing with sensitive data, there is a consequent degradation in the classification accuracy due to the introduced perturbations in the classifier weights. For such classifiers, counterfactual explanations need to be robust against the uncertainties in the SVM weights in order to ensure, with high confidence, that the classification of the data instance to be explained is different than its explanation. We model the uncertainties in the SVM weights through a random vector, and formulate the explanation problem as an optimization problem with probabilistic constraint. Subsequently, we characterize the problem’s deterministic equivalent and study its solution. For linear SVMs, the problem is a convex second-order cone program. For non-linear SVMs, the problem is non-convex. Thus, we propose a sub-optimal solution that is based on the bisection method. The results show that, contrary to non-robust explanations, the quality of explanations from the robust solution degrades with increasing privacy in order to guarantee a prespecified confidence level for correct classifications.
\end{abstract}

\section{Introduction}%
Despite their efficiency in solving complex problems, machine learning (ML) algorithms and models are seldom value-neutral to the extent that they include social and ethical values. Even when such values are integrated into the models they may be mandated by regulatory frameworks, such as traditional laws or policy documents. This paper aims to illustrate the relational nexus between social and ethical values in a technical context. This is done by focusing on three values advocated by the General Data Protection Regulation (GDPR) \cite{RegulationEU}, namely, \emph{explainability}, \emph{privacy}, and \emph{accuracy}.\footnote{References to these social values can be found in Recital 71, Article 25 and Article 5(1)(d) as expanded by the Article 29 Data Protection Working Party Adopted on 3 October 2017 respectively. It is also noteworthy that an in-depth discussion of what exactly these social values entail is beyond the confines of this paper.} What becomes apparent when attempting to transform the above social and ethical values from the natural language of the law into the mathematical language of ML algorithms is that this may be challenging and even technically unattainable. The above social and ethical values have been chosen as their transformation into ML rules clearly illuminates the challenge of aligning these competing social and ethical values promoted by the law into a technical format, a conclusion being that the simultaneous promotion of all these three values is potentially mathematically unattainable. 

\figurename~\ref{fig:outline} gives an overview on how the three mentioned social values are related within this work:
\begin{enumerate}
    \item Accuracy is targeted when learning an SVM classifier through the data from a dataset.
    \item Privacy is guaranteed by applying a privacy preserving mechanism. We will use the mechanism developed in \cite{Rubinstein2012} which guarantees a strong notion of privacy called \emph{differential privacy} \cite{Dwork2014}.
    \item Explainability of predictions for specific data instances enhances the transparency and the trust in the classifier predictions. We will consider \emph{counterfactual explanations} \cite{Wachter2017, molnar2019}, which is a class of explainability methods that quantify the necessary changes to a considered data instance in order to change its classification. Such explainability methods fall under the category of \emph{post hoc} interpretations \cite{Murdoch2019, Arrieta2020}, since they do not require model retraining.
\end{enumerate}

\begin{figure}[!t]
    \centering
    \includegraphics[width=.8\linewidth]{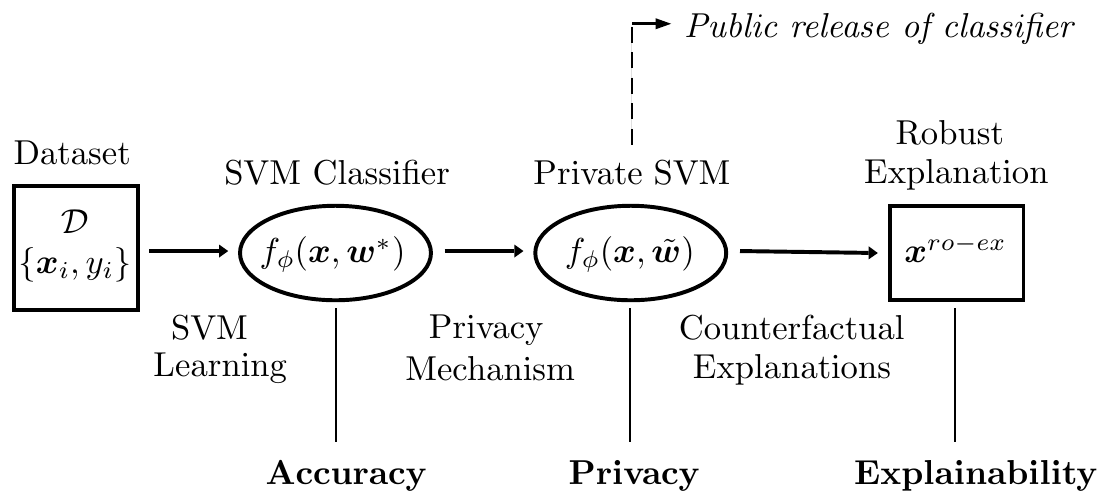}
    \caption{Illustration of the relationship between accuracy, privacy, and explainability associated with an SVM classifier trained on a dataset $\mathcal{D}$. A privacy mechanism is applied to preserve the privacy of the data before the public release of the SVM classifier. Explanations need to suitably take into account the introduced uncertainty through the privacy mechanism.}
\label{fig:outline}
\end{figure}

In this work, we will propose robust counterfactual explanations that \emph{exploit the characteristics of the SVM classifier and the applied privacy mechanism}. The privacy mechanism proposed in \cite{Rubinstein2012} perturbs the SVM weights through additive Laplace noise. As a result, privacy is achieved by establishing uncertainty about the true classifier weights. For constructing explanations, we suitably model the uncertainty in the SVM weights through random variables. Then, we formulate counterfactual explanations as a optimization problem with probabilistic constraints \cite{Shapiro2014}, and characterize its deterministic equivalent problem. For linear SVMs, the deterministic problem is a second-order cone program (SOCP) and can be solved efficiently. For the non-linear SVM case, the problem is non-convex. We therefore propose an efficient sub-optimal algorithm to find robust explanations utilizing the existence of class specific prototypes. The use of prototypes in the context of counterfactual explanations has been previously applied in \cite{Looveren2019} for the purpose of ensuring that explanations lie within the data domain and can be computed efficiently.

Experimental results illustrate the trade-offs between accuracy, privacy, and explainability using the UCI Breast Cancer Wisconsin Dataset \cite{Dua2017}. In particular, we show that, contrary to non-robust explanations, the changes to an instance required by its robust explanation increase with the level of privacy. This degradation in the explanation quality occurs in order to guarantee, with prespecified confidence, that the explanations are counterfactuals, i.e., their classification is different than the instance, with respect to the non-private and unknown classifier. To the best of our knowledge, this is the first work to study the design of explanations for privacy preserving classifiers.

\section{Preliminaries}
In this section, we will describe the dataset and the SVM learning problem. Then, we will review the privacy preserving mechanism for SVM proposed in \cite{Rubinstein2012}.

\subsection{Dataset and SVM Classification}
Consider a dataset $\mathcal{D}$ consisting of a collection of $n$ tuples 
\begin{equation}
(\mat{x}_i,y_i), \quad i=1,\ldots,n,
\end{equation}
where each tuple $(\mat{x}_i,y_i)$ consists of a \emph{features vector} $\mat{x}_i \in \mathbb{R}^L$ and its associated \emph{class label} $y_i \in \{-1,1\}$. 

Dataset $\mathcal{D}$ is used to learn an SVM classifier \cite[Ch. 12]{Hastie2009} that can efficiently separate the two classes of data points through a separating hyperplane. The optimization problem for the SVM with hinge loss and parameter $C \geq 0$, is:
\begin{align}\label{eq:svm}
\minimize_{\mat{w}\in\mathbb{R}^F}  & ~~ \frac{1}{2}\snorm{\mat{w}} + \frac{C}{n}\sum_{i=1}^n \ebr{1-y_i \f(\mat{x}_i,\mat{w})}_+,
\end{align}
where the weights $\mat{w}$ geometrically correspond to the vector perpendicular to the separating hyperplane, $[a]_+ := \max\{0,a\}$, and $\f$ is the \emph{classifier} function:
\begin{equation}\label{eq:f}
\f(\mat{x},\mat{w}) := \phi(\mat{x})^\T\mat{w}.
\end{equation}
Here, the \emph{feature mapping} $\phi: \mathbb{R}^L \rightarrow \mathbb{R}^{F}$, $F \geq L$, enlarges the feature space of the data points to improve the separability of the two classes of data points through a hyperplane \cite[Ch. 12.3]{Hastie2009}. We assume in this work that $F$ is finite.

The minimization problem defined in Eq. \eqref{eq:svm} can be formulated as a quadratic program and solved efficiently. Let, $\mat{w}^*$ be the optimal solution to this problem, %\eqref{eq:svm},
then the \emph{binary classification} of a given data point $\mat{x}$ is
\begin{equation}\label{eq:h}
\h(\mat{x},\mat{w}^*) := \mathrm{sign}[\f(\mat{x},\mat{w}^* )],
\end{equation}
where $\mathrm{sign}[a] = -1$, if $a < 0$ and $\mathrm{sign}[a] = 1$, if $a \geq 0$.

If the number of features $F$ is larger than the number of available data points $n$, then it is more efficient to solve the dual problem of Eq. \eqref{eq:svm}:
\begin{subequations}\label{eq:dual}
\begin{align}\label{eq:dual_obj}
    \maximize_{\mat{\alpha}\in\mathbb{R}^n} & ~~ \sum_i^n \alpha_i - \frac{1}{2} \sum_i^n \sum_j^n \alpha_i \alpha_j y_i y_j \phi(\mat{x}_i)^\T \phi(\mat{x}_j)\\
   s.t. & ~~ 0 \leq \alpha_i \leq {C}/{n}, \quad i = 1,\ldots,n,
\end{align}
\end{subequations}
since then the number of optimization variables is smaller. Another advantage in solving the dual is the possibility to apply the kernel trick \cite[Ch. 12.3]{Hastie2009}, where a kernel function $k(\mat{x}_i,\mat{x}_j)$ replaces $\phi(\mat{x}_i)^\T \phi(\mat{x}_j)$ in Eq. \eqref{eq:dual_obj}.

The optimal solution of the primal problem in Eq. \eqref{eq:svm} is then computed from the solution of the dual problem in Eq. \eqref{eq:dual} as:
\begin{align}
    \mat{w}^* = \sum_i^n \alpha_i^* y_i \phi(\mat{x}_i).
\end{align}

From Eq. \eqref{eq:f} it can be observed that in order to perform SVM classification, all we need is $\mat{w}^*$ and the feature mapping~$\phi$. In applications where the dataset includes sensitive information, the public release of the SVM classifier may lead to privacy breaches through the information in $\mat{w}^*$. Therefore, it is required to apply a privacy preserving mechanism before the public release of the classifier, as is illustrated in \figurename~\ref{fig:outline}.

\subsection{Privacy Preserving Mechanism for SVM}

In this section, we will describe the privacy preserving mechanism proposed in \cite[Sec. 3]{Rubinstein2012} for SVMs with finite dimensional feature mappings. The proposed mechanism that guarantees differential privacy essentially perturbs the SVM optimal weights $\mat{w}^* \in \mathbb{R}^F$ by additive Laplace noise.

More formally, let $M: \mathfrak{D} \rightarrow \mathcal{R} $ be a \emph{randomized mechanism}, where $\mathfrak{D}$ is the set of all datasets and $\mathcal{R}$ is the response set of the mechanism $M$ (defined as the solution space of the dual problem in Eq. \eqref{eq:dual}). Define \emph{neighboring datasets} as the datasets in $\mathfrak{D}$ that differ by one data point entry. Then, for a given $\beta>0$, a randomized mechanism $M$ provides {$\beta$-differential privacy} \cite[Def. 2.4]{Dwork2014} if for any two neighboring datasets $\mathcal{D}_1, \mathcal{D}_2 \in \mathfrak{D}$ and all response subsets $\mathcal{S} \subseteq \mathcal{R}$ it holds
\begin{equation}\label{def:privacy}
\prob{{M}(\mathcal{D}_1) \in \mathcal{S}} \leq \exp({\beta}) ~ \prob{{M}(\mathcal{D}_2) \in \mathcal{S}}.
\end{equation}
%holds for all {neighboring datasets} $\mathcal{D}_1$ and $\mathcal{D}_2$ in $\mathfrak{D}$, and all response subsets $\mathcal{S} \subseteq \mathcal{W}$.

From \cite[Thm. 10]{Rubinstein2012}, the perturbed SVM weight vector
\begin{equation}\label{eq:privacy_mechanism}
    \tilde{\mat{w}} := \mat{w}^* + \mat{\mu},
\end{equation}
where $\mat{\mu}$ is a vector of iid Laplace random variables
\begin{equation}\label{eq:mu}
    \mu_i \sim \textrm{Lap}(0,\lambda), i = 1,\ldots,F,
\end{equation} 
achieves $\beta-$differential privacy for 
\begin{equation}
    \lambda \geq 4 C \kappa \sqrt{F}/ (\beta n),
\end{equation}
where $\kappa$ satisfies $\phi(\mat{x})^\T \phi(\mat{x}) \leq \kappa^2$ for all $\mat{x}\in\mathbb{R}^L$.

By perturbing the optimal weight vector, the accuracy of the SVM classifier will be degraded. For this purpose, it is important to deliver guarantees on the classification accuracy by upper bounding the noise scale $\lambda$. This is done in \cite{Rubinstein2012} by introducing a condition called $(\epsilon,\delta)$-useful mechanism:
\begin{equation}\label{eq:useful_mechanism}
    \prob{ \sup_{\mat{x}\in \mathcal{X}}\abs{\f(\mat{x},\tilde{\mat{w}} ) - \f(\mat{x},\mat{w}^* )} \leq \epsilon } \geq 1-\delta.
\end{equation}
where $\epsilon > 0, \delta\in(0,1)$, and the set $\mathcal{X} \subseteq \mathbb{R}^L$ contains all the data points. From \cite[Thm. 11]{Rubinstein2012}, the condition above is satisfied for $0 \leq \lambda \leq \frac{\epsilon}{2 \Phi (F - \ln \delta) }$, where $\abs{\phi_i(\mat{x})} \leq \Phi$ for all $\mat{x} \in \mathcal{X}$, $i=1,\ldots,F$, and some $\Phi > 0$.
%\begin{equation}
%    k(\mat{x}_i,\mat{x}_j) = \mathrm{exp}\left( - \gamma %{\norm{\mat{x}_i-\mat{x}_j}_2^2}\right),
%\end{equation}

In the following, we will assume that the following information is publicly available: the SVM weights $\tilde{\mat{w}}$, the data-independent details for constructing $\phi$, and the noise scale $\lambda$.

\section{Counterfactual Explanation}

\begin{figure}
    \centering
    \includegraphics[width=0.7\linewidth]{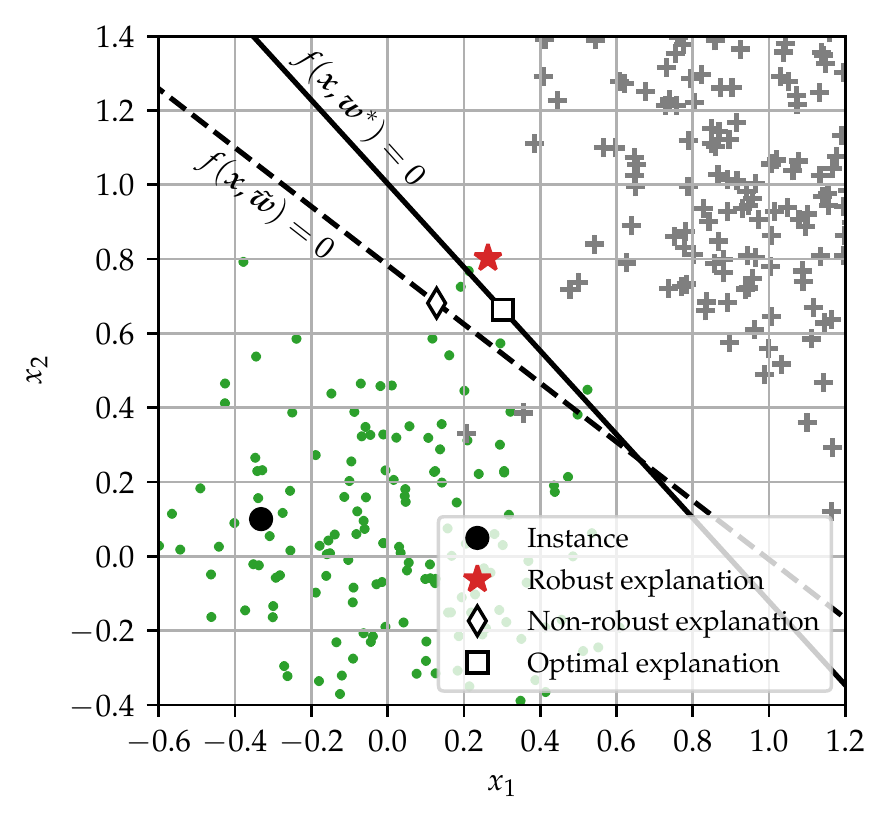}
    \caption{Illustration for linear classification with SVM and private SVM, and the different associated explanations using the Euclidean norm as distance measure. The data points are generated from two bivariate Guassian distributions with means $[0,0]$ and $[1,1]$, and same covariance $0.1 \mat{I}$.}
    \label{fig:linear}
\end{figure}

The concept of counterfactual explanations was proposed in \cite[Eq. 2]{Wachter2017} for general ML classifiers. In the following definition we use the original definition, but formulated for the SVM.

\begin{definition}[Counterfactual Explanation]
Given an SVM classifier with weight vector $\mat{w}$, a counterfactual explanation for the classification $y' = h(\mat{x}',\mat{w})$ of a given data instance $\mat{x}'$ is the solution of the following problem
\begin{subequations}\label{prob:explain}
\begin{align}
\minimize_{\mat{x} \in \mathbb{R}^L} & ~~  d(\mat{x},\mat{x}') \\ \label{eq:constraint_reformulation} s.t. & ~~ y' \f(\mat{x},\mat{w}) \leq 0,
\end{align}
\end{subequations}
where $d(\mat{x},\mat{x}')$ is a distance between $\mat{x}$ and $\mat{x}'$, assumed to be convex in the first argument, and $\f(\mat{x},\mat{w})$ is defined in Eq. \eqref{eq:f}.
\end{definition}
%\footnote{For any $\mat{x},\mat{z}\in\mathbb{R}^L$ and $\eta \in [0,1]$, the function $d$ satisfies: $$d(\eta \mat{x} + (1-\eta)\mat{z},\mat{x}') \leq \eta d(\mat{x},\mat{x}') + (1-\eta)d(\mat{z},\mat{x}').$$}

The solution of problem \eqref{prob:explain} is the closest point to $\mat{x}'$, according to the distance $d$, with the constraint that its class is different than $y'$. Notice that since we consider binary classification, the constraint in Eq. \eqref{eq:constraint_reformulation} is equivalent to $h(\mat{x},\mat{w}) \neq y'$. 

In \figurename~\ref{fig:linear}, we illustrate different counterfactual explanations for the linear SVM classifiers with optimal weights $\mat{w}^*$ and perturbed weights $\tilde{\mat{w}}$. The optimal and non-robust explanations are the closest points to the instance that lie on the respective decision boundaries. It can be seen that the non-robust explanation is closer to the instance compared to the optimal explanation and thus also has the same classification as the instance when using the optimal classifier. Hence, the non-robust explanation may not be credible. We will study robust explanations next that will take into account the uncertainty in the perturbed weights.

\subsection{Robust Explanations}

%%%% Motivation
The private SVM mechanism releases noisy versions of the optimal $\mat{w}^*$ according to Eq. \eqref{eq:privacy_mechanism}. Thus, there exists uncertainty about the correctness of the classification with $\tilde{\mat{w}}$, which diminishes the effectiveness of the counterfactual explanation unless this uncertainty is taken into account. Therefore, we will model the uncertainty about $\mat{w}^*$ through the random vector $\mat{\xi} = \tilde{\mat{w}} - \mat{\mu}$. From Eq. \eqref{eq:mu}, it follows that
\begin{equation}\label{eq:xi}
\mat{\xi} \sim \fun{mvLap}\pp{\tilde{\mat{w}}, 2\lambda^2\mat{I}},
\end{equation}
where $\fun{mvLap}\pp{\mat{l},\mat{\Sigma}}$ is the multivariate Laplace distribution with location $\mat{l}$ and covariance matrix $\mat{\Sigma}$. Subsequently, we will replace the constraint in \eqref{eq:constraint_reformulation} with the probabilistic constraint
\begin{align}\label{eq:prob_constraint}
\prob{y' \f(\mat{x},\mat{\xi}) \leq 0} \geq p.
\end{align}
The probability $p$ in Eq. \eqref{eq:prob_constraint} is typically selected close to $1$. In the following, we will characterize the deterministic equivalent of the constraint in Eq. \eqref{eq:prob_constraint}.

\begin{proposition}\label{prop:chanceConst} The robust counterfactual explanation problem%
\begin{subequations}\label{prob:robust_explain}
\begin{align}
\minimize_{\mat{x} \in \mathbb{R}^L} & ~~  d(\mat{x},\mat{x}')\\ \label{prob:robust_explain_constraint}
s.t. & ~~ \prob{y' \f(\mat{x},\mat{\xi}) \leq 0} \geq p,
\end{align}
\end{subequations}
with $p\in [1/2,1]$ is equivalent to
\begin{subequations}\label{prob:robust_explain1}
\begin{align}
\minimize_{\mat{x} \in \mathbb{R}^L} & ~~ d(\mat{x}, \mat{x}') \\ \label{eq:prob_constraint_ref}
 s.t. & ~~ y'\phi(\mat{x})^\T \tilde{\mat{w}} - \lambda \sqrt{2} \ln(2(1-p)) \norm{\phi(\mat{x})} \leq 0.
\end{align}
\end{subequations}
\end{proposition}
\begin{proof} We need to prove that the probabilistic constraint in Eq. \eqref{prob:robust_explain_constraint} can be reformulated to Eq. \eqref{eq:prob_constraint_ref}. From \eqref{eq:xi}, the multivariate Laplace distribution $\fun{mvLap}\pp{\tilde{\mat{w}}, 2\lambda^2\mat{I}}$ is \emph{symmetric} since the variance does not depend on the mean. A symmetric multivariate Laplace distribution is \emph{elliptically symmetric} \cite{Kotz2001}. Consequently, the structure of Eq. \eqref{eq:prob_constraint_ref} follows from \cite[Lemma 2.2]{Henrion2007}, and for the multivariate Laplace distribution, the derivation follows similar steps as in \cite[Ex. 2.2]{Peng2019}.
%Then, the distribution $\fun{mvLap}\pp{\tilde{\mat{w}}, 2\lambda^2\mat{I}}$ has the characteristic function
%\begin{equation}
%    \Psi(\mat{t}) := \frac{\exp(i \tilde{\mat{w}}^\T \mat{t})}{1 + \frac{1}{2} \mat{t}^\T (2\lambda^2\mat{I}) \mat{t}}.
%\end{equation}
\end{proof}

The constraint in Eq. \eqref{eq:prob_constraint_ref} includes two terms. The first term is the same as in Eq. \eqref{eq:constraint_reformulation} and requires that the solution of the problem has a different class than $y'$. The second term establishes robustness by enforcing stronger confidence in the SVM prediction, i.e., larger $\abs{\phi(\mat{x})^\T \tilde{\mat{w}}}$. Notice that for $p=0.5$, the second term is zero and Eq. \eqref{eq:prob_constraint_ref} becomes the same as that in the non-robust case in Eq. \eqref{eq:constraint_reformulation}. 

For linear SVM, i.e., $\phi(\mat{x}) = \mat{x}$, problem \eqref{prob:robust_explain1} is rewritten as 
\begin{align}
\minimize_{\mat{x} \in \mathbb{R}^L} ~~ d(\mat{x}, \mat{x}') \quad s.t.  ~~ \norm{\mat{x}} \leq \frac{y'}{\lambda \sqrt{2} \ln(2(1-p))}\mat{x}^\T \tilde{\mat{w}},
\end{align}
which is a SOCP problem and can be solved efficiently using convex optimization solvers. The implementations for this letter have been done using CVXPY \cite{diamond2016cvxpy, agrawal2018rewriting}, and the explanations plotted in \figurename~\ref{fig:linear} were found by solving this problem. 

The general problem in \eqref{prob:robust_explain1} is however not convex. Therefore, we will consider next finding a suboptimal solution that can be computed efficiently. Define the function
\begin{equation}
g(\mat{x}):= y' \phi(\mat{x})^\T \tilde{\mat{w}} - \lambda \sqrt{2} \ln(2(1-p)) \norm{\phi(\mat{x})},
\end{equation}
which is the left hand side of Eq. \eqref{eq:prob_constraint_ref}. A \emph{root} for the function $g$ would qualify as a robust explanation since it satisfies the constraint in \eqref{eq:prob_constraint_ref} with equality. In order to find a root for $g$, we will use the \emph{bisection method}. As a prerequisite, this method requires the knowledge of two input values to the function that give opposite signs. Clearly, for the given data instance $\mat{x}'$, $g(\mat{x}')$ is positive. The second required input vector should necessarily be of opposite class to $\mat{x}'$ in order for $g$ to be negative. We will discuss next the availability of such input that we will here refer to as a \emph{prototype} \cite{Looveren2019}.

Unlike in \cite{Looveren2019}, we do not have access to test data to construct these prototypes due to privacy issues. However, we argue that if we consider prototypes as representatives of their classes, the ``domain expert'' that provides the explanations should be able to estimate these by knowing the characteristics of the data for each class. If this is not the case, we assume that the prototypes can be constructed by generating random data instances and studying their classification. Let the prototypes for class $1$ and $-1$ be $\mat{z}_1$ and $\mat{z}_{-1}$, respectively. In the process of finding the prototypes, it is desired that the classification for these points has sufficient confidence, i.e.,
\begin{equation}\label{eq:rob_proto}
    \abs{\f(\mat{z}_y,\tilde{\mat{w}})} \geq -\lambda \sqrt{2} \ln(2(1-p)) \norm{\phi(\mat{z}_y)},\quad y \in \{1,-1\}.
\end{equation}

The steps for the bisection method are described in Algorithm~\ref{alg:bisection}. The lower and upper bounds for bisection are initialized according to the given data instance and the prototype from the opposite class, respectively. In each iteration we check the classification of the midpoint of the interval between the upper and lower bounds. If this class is the same as the lower bound, then we replace the lower bound by the midpoint. Otherwise, we replace the upper bound. These steps are performed until the distance between the upper and lower bounds is lower than the threshold $\epsilon$. The algorithm has linear convergence since the distance between the bounds is halved in each iteration.

\begin{algorithm}[t!]
   \caption{Bisection method for finding an explanation}
   \label{alg:bisection}
\begin{algorithmic}[1]
\STATE {\bfseries Input:} Instance and classification $(\mat{x}', y')$, prototype $\mat{z}_{-y'}$.
\STATE {\bfseries Initialize:}  ${\mat{x}^{ub}} = \mat{z}_{-y'}, {\mat{x}^{lb}} = \mat{x}'$
\WHILE{$\norm{\mat{x}^{ub} - \mat{x}^{lb}} > \epsilon$}
    \STATE $\mat{x} \leftarrow (\mat{x}^{ub} + {\mat{x}^{lb}})/2$
    \IF{$g(\mat{x})<0$}
        \STATE $\mat{x}^{ub} \leftarrow \mat{x}$
    \ELSE
        \STATE $\mat{x}^{lb} \leftarrow \mat{x}$
    \ENDIF
\ENDWHILE
    \STATE {\bfseries Output:} $\mat{x}^{ro-ex} \leftarrow \mat{x}$.
\end{algorithmic}
\end{algorithm}

\section{Experimental Results }
We illustrate our approach by using the publicly available UCI Breast Cancer Wisconsin (Diagnostic) dataset \cite{Dua2017}. The dataset includes $569$ instances, each with $30$ features and the binary diagnosis: benign (class $-1$) or malignant (class $1$). The code to reproduce all the figures is available at \cite{Mochaourab2021}.

We randomly split the dataset once into a training (70\% of total) and a test set (30\% of total). The results were qualitatively similar for different random splits of the dataset with the same splitting ratio. Moreover, we normalize the training data to have zero mean and unit variance, and the calculated normalization parameters are applied to the test data. Next, a feature mapping $\phi$ is generated using the Radial Basis Function (RBF) kernel approximation in \cite{Rahimi2007} with dimensions $F=100$.\footnote{Note, that in this work we have assumed finite dimensional feature mappings and hence we do not explicitly consider the approximation error in the feature mapping in relation to using the RBF kernel as is done in \cite[Section 4]{Rubinstein2012} for the general case of translation-invariant kernels.} For the implementation of the feature mapping, we have used the library in \cite{Atarashi}. The SVM classifiers learned for the plots are trained using the training set and their performance  measured on the test set. The distance function used for the counterfactuals in Eq. \eqref{prob:robust_explain1}, is the Eucleadian norm, i.e., $d(\mat{x},\mat{x}') = \norm{\mat{x} - \mat{x}'}$. The prototypes are selected to be the data mean of each class.

\smallskip
\noindent
\textbf{Accuracy vs privacy.}
\figurename~\ref{fig:pri_acc_tradeoff}, depicts the trade-off between average accuracy and privacy of the private SVM, with an average taken over $10^4$ random realizations of Laplace noise. The dashed line corresponds to the non-private case in which the SVM weights are not perturbed with noise. The average accuracy for the private SVM is lowest ($\approx 0.5$) for high privacy levels (very small $\beta$), and monotonically increases with $\beta$ to eventually converge to the non-private SVM average performance.

\begin{figure}[!t]
    \centering
    \includegraphics[width=.65\linewidth]{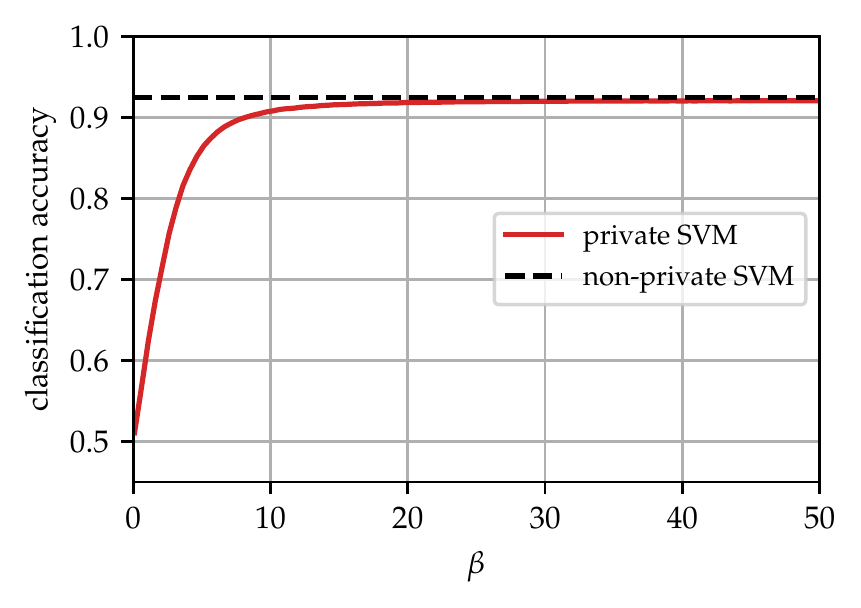}
    \caption{Average accuracy of private SVM for different values of $\beta$ in \eqref{def:privacy}. The confidence value is set to $p=0.9$ for the probabilistic constraint in \eqref{eq:prob_constraint}.}
    \label{fig:pri_acc_tradeoff}
\end{figure}%

\smallskip
\noindent
\textbf{Convergence and explainability.}
We set $\beta = 5, p=0.9$. Then, we select a random instance $\mat{x}'$ from the test set with label $y' = 1$ (malignant), and apply Algorithm \ref{alg:bisection} to calculate an explanation $\mat{x}^{ro-ex}$ for its classification. \figurename~\ref{fig:convergence} shows the low number of iterations needed for Algorithm \ref{alg:bisection} to converge. The found explanation $\mat{x}^{ro-ex}$ quantifies the changes to each feature of $\mat{x}'$ in order to change the classifier prediction. \figurename~\ref{fig:explanation} shows these changes normalized over the instance's feature values. For example, for the selected instance, the explanation shows that feature number $19$ needs to be increased by around half its value, while several other feature values need to be halved in order to alter the prediction from malignant to benign.

\smallskip
\noindent
\textbf{Explainability vs privacy.}
The average distance between the counterfactual explanation and the instance is calculated depending on $\beta$ (for $p=0.9$) in \figurename~\ref{fig:dist_inst}, and depending on $p$ (for $\beta = 0.5$) in \figurename~\ref{fig:dist_inst_p}. This average distance for robust counterfactual explanations is high for small values of $\beta$, as is shown in \figurename~\ref{fig:dist_inst}. This is due to the large uncertainty through the large noise variance. The non-robust explanation has similar distances as for non-private SVM since the noise has zero mean. In \figurename~\ref{fig:dist_inst_p}, the effects of $p$ on the average distance are shown. For $p=0.5$, robust explanation is identical with non-robust explanation as mentioned before. For large confidence values $p$, the robust explanation converges to the prototype data point, and is furthest away from the instance.

%\textbf{Closeness to original instance.}
Clearly, it is desirable to find counterfactual explanations that are as close as possible to the instance to explain. Still, as we observe in \figurename~\ref{fig:dist_inst_all}, robust explanations are further away compared to the non-robust explanations, showing that privacy degrades the \emph{quality} of explanations. The reason for that is, non-robust explanations violate the constraint in Eq. \eqref{eq:constraint_reformulation} with probability $0.5$, while the robust explanations violate this constraint with probability $p$ (which we here set to $0.9$). This constraint violation is further studied in \figurename~\ref{fig:dist_all}. In \figurename~\ref{fig:dist_non_robust} and \figurename~\ref{fig:dist_robust}, the summary statistics for the left hand side of \eqref{eq:constraint_reformulation} are plotted for the non-robust and robust explanations, respectively. These plots highlight the importance for considering robust explanations. Notice that the flattening of the $50$-th percentile curve (\figurename~\ref{fig:dist_robust}) for $\beta$ less than around $4$ is due to the convergence of the explanation to the prototype.

\begin{figure}[!t]
    \centering
    \begin{subfigure}[b]{0.65\linewidth}
    \includegraphics[width=\linewidth]{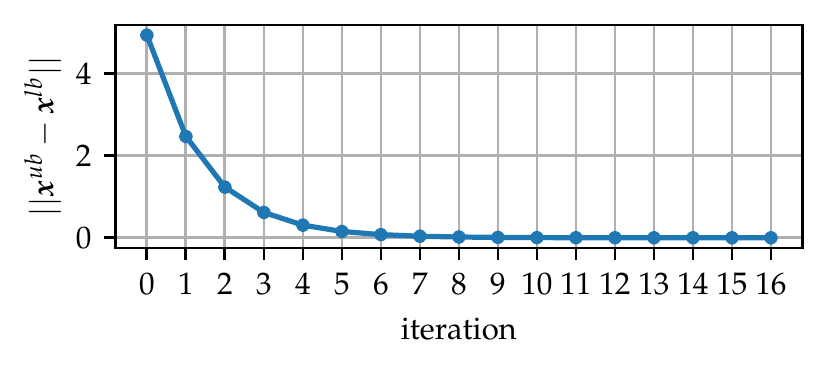}
    \caption{Convergence of Algorithm \ref{alg:bisection}.}
    \label{fig:convergence}
    \end{subfigure}
    \hfill
    \begin{subfigure}[b]{0.65\linewidth}
    \centering
    \includegraphics[width=\linewidth]{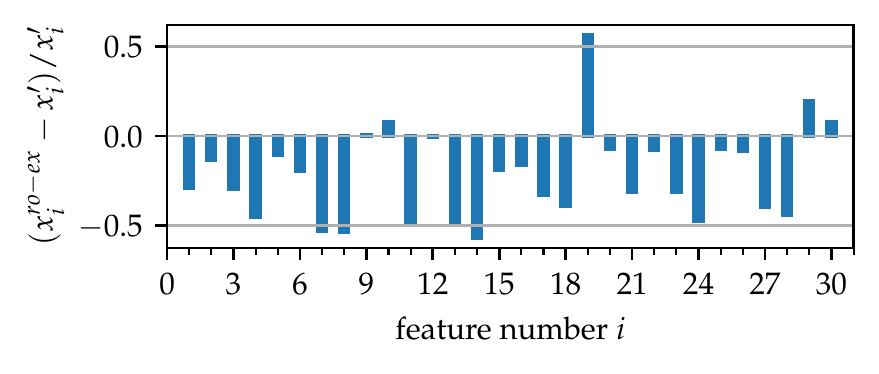}
    \caption{The counterfactual explanation is utilized to quantify the necessary changes in the instance's features in order to alter its SVM classifier prediction.}
    \label{fig:explanation}
    \end{subfigure}
    \caption{Convergence of Algorithm \ref{alg:bisection} and the explanation of an instance.}
    \label{fig:convergence_to_explanation}
\end{figure}

\section{Conclusions}

The above findings highlight the difficulties associated with embedding the social and ethical values mandated by regulatory instruments into ML algorithms. An ensuing conclusion is that a conscious decision may be required to promote one social value at the expense of another, the context in which the technology is being operated potentially being a deciding factor. These issues are highlighted in this work through the study of robust counterfactual explanations for privacy preserving SVMs. After suitably modelling the problem, we have studies its solution for linear and non-linear SVMSs. While the problem can be solved optimally for the linear case, we have proposed an efficient suboptimal solution based on the bisection method for the general case. The advantages for using robust explanations are illustrated on the UCI Breast Cancer Wisconsin dataset. In particular, our robustness model guarantees a tuneable level of confidence that the counterfactuals are of opposite class to that of the instance we want to explain.

\section{Acknowledgments}
This work has been supported by KTH Digital Futures within the project ``EXTREMUM: Explainable and Ethical Machine Learning for Knowledge Discovery from Medical Data Sources'' (\texttt{https://www.digitalfutures.kth.se}).

\begin{figure*}[!ht]
    \centering
    \begin{subfigure}[b]{0.65\textwidth}
         \centering
        \includegraphics[width=\textwidth]{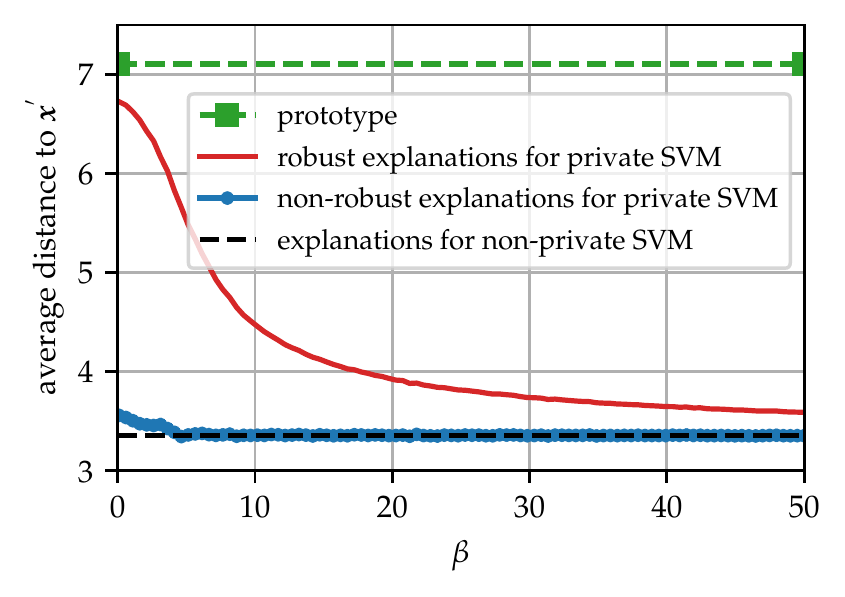}
        \caption{$p=0.9$}
        \label{fig:dist_inst}
    \end{subfigure}
    \hfill
    \begin{subfigure}[b]{0.65\textwidth}
         \centering
        \includegraphics[width=\textwidth]{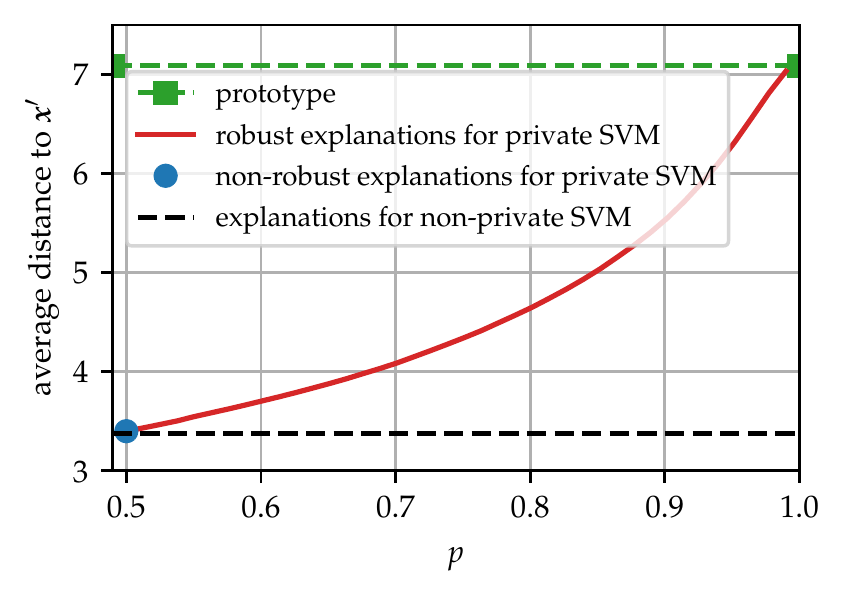}
        \caption{$\beta=5$}
        \label{fig:dist_inst_p}
    \end{subfigure}
    \caption{Robust explanations lie between the instance $\mat{x}'$ and the selected prototype. This distance decreases for larger $\beta$ and increases for larger probability $p$.}
    \label{fig:dist_inst_all}
    \vspace{-2mm}
\end{figure*}
\begin{figure*}[t]
    \centering
    \begin{subfigure}[b]{0.65\textwidth}
        \centering
        \includegraphics[width=\textwidth]{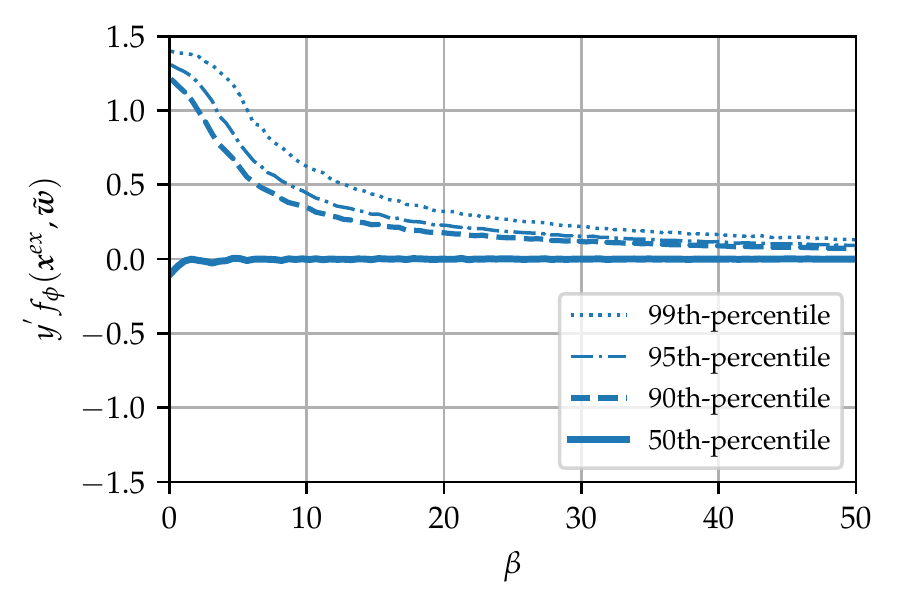}
        \caption{Non-robust explanations}
        \label{fig:dist_non_robust}
    \end{subfigure}
    \hfill
    \begin{subfigure}[b]{0.65\textwidth}
        \centering
        \includegraphics[width=\textwidth]{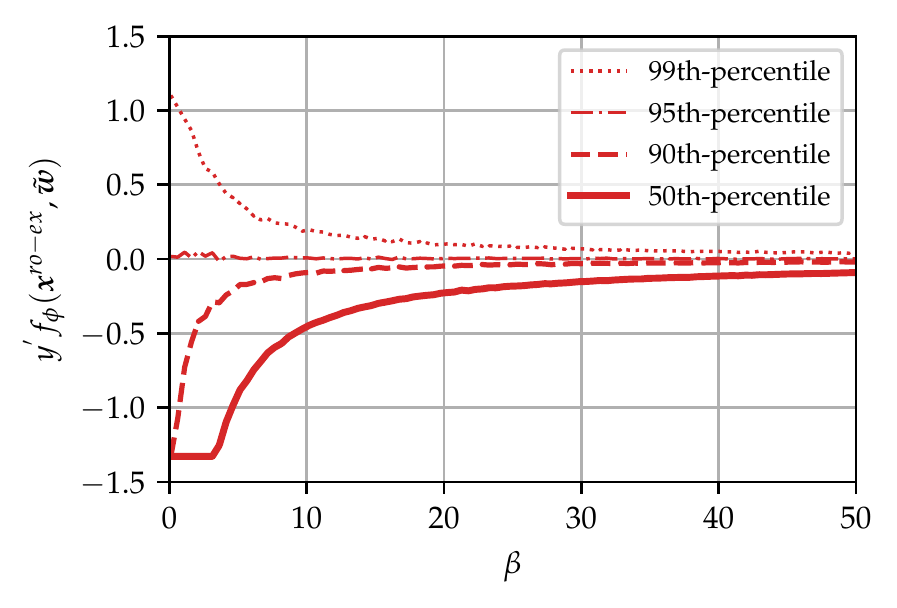}
        \caption{Robust explanations}
        \label{fig:dist_robust}
    \end{subfigure}
    \caption{Comparison between robust and non-robust explanations on the violation of the constraint in \eqref{eq:constraint_reformulation} for private SVM.}
    \label{fig:dist_all}
\vspace{-.775mm}
\end{figure*}

\bibliographystyle{elsarticle-num}
\bibliography{references}

\end{document}